\documentclass{article}
\usepackage{PRIMEarxiv}
\usepackage[utf8]{inputenc}
\usepackage[T1]{fontenc}
\usepackage{hyperref}   
\usepackage{url}        
\usepackage{booktabs}   
\usepackage{amsfonts}   
\usepackage{nicefrac}   
\usepackage{microtype}  
\usepackage{lipsum}
\usepackage{fancyhdr}   
\usepackage{graphicx}   
\title{Identifiability and Order-Dimension Limits of In-Context Learning on Partial Orders}

\author{
  Faizanuddin Ansari \\
  Indian Statistical Institute \\
  Kolkata, India \\ 
  \href{mailto:faizanansari541@gmail.com?subject=[From arXiv] PosetICL Paper}{\texttt{faizanansari541@gmail.com}}\\
   \And
  Debanjan Dutta \\
  Indian Statistical Institute \\
  Kolkata, India \\
   \And
  Swagatam Das \\
  Indian Statistical Institute \\
  Kolkata, India \\
  \href{mailto:swagatam.das@isical.ac.in?subject=[From arXiv] PosetICL Paper}{\texttt{swagatam.das@isical.ac.in}}
}

\usepackage{booktabs}
\usepackage{amsmath,amssymb,amsthm}
\usepackage{natbib, caption, wrapfig, xcolor}

\definecolor{mygreen}{HTML}{2CA02C}
\hypersetup{
        pdfinfo={
            Author={Faizanuddin Ansari},
            Title={Identifiability and Order-Dimension Limits of In-Context Learning on Partial Orders},
            Subject={arXiving 2026 Work},
            Keywords={Poset, ICL},
            CreationDate={09/08/2026--19:56:00},
            ModDate={\the\day/\the\month/\the\year\space},
            Creator={PdfLaTeX v3.14},
            Producer={LaTeX},
        },
        colorlinks=true,
        urlcolor=black,
        citecolor=mygreen
    }

\newtheorem{theorem}{Theorem}
\newtheorem{proposition}[theorem]{Proposition}
\newtheorem{lemma}[theorem]{Lemma}
\newtheorem{observation}[theorem]{Observation}
\newtheorem{corollary}[theorem]{Corollary}
\newtheorem{definition}[theorem]{Definition}

\newcommand{\D}{\mathcal{D}}
\newcommand{\V}{\mathcal{V}}
\newcommand{\B}{\mathcal{B}}

\newcommand{\TC}{\operatorname{TC}}
\newcommand{\Reach}{\operatorname{Reach}}
\newcommand{\dimposet}{\operatorname{dim}}
\newcommand{\height}{\operatorname{height}}
\newcommand{\width}{\operatorname{width}}
\newcommand{\ind}{\mathbf{1}}
\newcommand{\notR}{\not\mathrel{R}}

\usepackage{subfiles}
\makeatletter
\newcommand{\makeappendixtitle}[1][Appendix]{%
  \par
  \begingroup
    \renewcommand{\thefootnote}{\fnsymbol{footnote}}
    \renewcommand{\@makefnmark}{\hbox to \z@{$^{\@thefnmark}$\hss}}
    \long\def\@makefntext##1{%
      \parindent 1em\noindent
      \hbox to 1.8em{\hss $\m@th ^{\@thefnmark}$}##1%
    }
    \thispagestyle{empty}
    \vbox{%
      \hsize\textwidth
      \linewidth\hsize
      \vskip 0.1in
      \@toptitlebar
      \centering
      {\LARGE\sc \@title\par}
      \@bottomtitlebar
      \vskip 0.1in
      \def\And{%
        \end{tabular}\hfil\linebreak[0]\hfil%
        \begin{tabular}[t]{c}\bf\rule{\z@}{24\p@}\ignorespaces%
      }
      \def\AND{%
        \end{tabular}\hfil\linebreak[4]\hfil%
        \begin{tabular}[t]{c}\bf\rule{\z@}{24\p@}\ignorespaces%
      }
      \begin{tabular}[t]{c}\bf\rule{\z@}{24\p@}\@author\end{tabular}%
      \vskip 0.3in \@minus 0.1in \center{} \vskip 0.1in
    }%
    \@thanks
  \endgroup
}
\makeatother

\begin{document}
\maketitle

\begin{abstract}
    In-context learning is commonly formalized as inference from examples of a function. Partial orders instead combine transitivity, antisymmetry, and incomparability, so a finite prompt may not determine a queried comparison. We develop a theory of in-context learning on partial orders that separates logical identifiability, prompt teaching cost, structural complexity, and the exact capacity of a formal coordinate-decoder class. A version-space semantics makes background knowledge and open- versus closed-world assumptions explicit. For finite open-world prompts with positive and negative comparisons, we prove an exact completion trichotomy: after taking the reflexive transitive closure of the positive demonstrations, a query is forced true, forced false because every true completion creates a cycle or violates a negative demonstration, or remains genuinely ambiguous. For a known $n$-element universe, we characterize the open-world teaching number as the number of covers plus a blocker-set hitting number, prove that its maximum over all $n$-element posets is $n(n-1)$ and is uniquely attained by the antichain, and identify the blocker term as the exact cost of open-world rather than complete-Hasse semantics. We formalize prompt-dependent $s$-coordinate decoders and use the classical coordinate-order equivalence to obtain an exact representation boundary: dimension at most $s$ is necessary and sufficient, while width at most $s$ is a convenient sufficient condition. 
\end{abstract}

\section{Introduction}
    Large language models can adapt to a task from demonstrations in the prompt without parameter updates, a capability known as in-context learning (ICL) \cite{brown2020language}. Much of the theory studies prompts sampled from an unknown function and asks the model to predict the function value at a new input \cite{garg2022what,akyurek2023what,dai2023why,guo2024how,bhattamishra2024understanding}. Statistical analyses characterize Bayes-optimal or information-limited ICL under generative assumptions \cite{jeon2024information}. These formulations are valuable, but many reasoning tasks are relational rather than single-valued.

    A partial order $\preceq$ is reflexive, antisymmetric, and transitive and may leave pairs incomparable. Finite posets are represented by Hasse diagrams, and comparability is reachability in the reflexive transitive closure of the cover graph \cite{dushnik1941posets,trotter1992dimension}. Posets therefore expose two difficulties that function-learning abstractions can hide. First, missing evidence is not the same as evidence of incomparability. Second, even when the relation is fully specified, its structure may require several independent ordering coordinates or long transitive certificates.

    A prior empirical study introduced prompts for linear order and divisibility and reported performance saturation on current language models \cite{dutta-etal-2025-assessing}. This study asks what a relational prompt logically determines, how many labels are required to teach a finite poset, and which posets admit exact decoding by a formally specified coordinate-order representation. Our framework includes a background theory $\B$. This matters because a prompt naming ``less than on natural numbers'' may permit semantic recall from pretraining, whereas an abstract relation symbol with only poset axioms permits many completions. We distinguish these cases through the version space \cite{mitchell1977version}. We also separate open-world semantics, where unmentioned relations may hold, from closed-world semantics, where a displayed finite Hasse diagram is declared complete \cite{reiter1978closed}.

    \paragraph{Contributions.}
        (1) We formalize relational ICL using a version space of posets on a fixed known universe, consistent with demonstrations and background knowledge. (2) We prove an exact true/false/unknown completion theorem for finite open-world prompts with positive and negative demonstrations and give its per-query classification cost. (3) We characterize optimal open-world teaching prompts through cover labels and a blocker-set hitting problem, derive exact chain and antichain values, and prove the tight class maximum $n(n-1)$. (4) We organize structural difficulty using height, width, order dimension, and positive and negative certificates, and establish an exact capability boundary for prompt-dependent monotone coordinate decoders.

    \paragraph{Overview.}
        The paper moves from logic to teaching, representation, and certification. We first formalize relational prompts through a fixed universe, a background theory, and a version space, making the distinction between open- and closed-world semantics explicit. We then characterize when a queried comparison is forced true, forced false, or genuinely ambiguous, and illustrate this trichotomy through an exhaustive enumeration of four-element posets. Next, we determine the labels required to teach an entire finite poset and isolate the additional cost created by open-world semantics. Finally, we compare representative poset families, establish the exact capability boundary for coordinate-order decoders, analyze positive and negative certificates, and conclude with the implications, scope, and limitations of the framework.

\section{Related Work and Positioning}
    \paragraph{ICL mechanisms and limits.}
        Function-learning, implicit-optimization, Bayesian, and task-representation accounts of ICL are developed by \citeauthor{garg2022what} (\citeyear{garg2022what}), \citeauthor{akyurek2023what} (\citeyear{akyurek2023what}), \citeauthor{dai2023why} (\citeyear{dai2023why}), \citeauthor{jeon2024information} (\citeyear{jeon2024information}), and \citeauthor{hendel2023task} (\citeyear{hendel2023task}). Recent mechanistic work finds task information in selected attention heads or low-dimensional activation subspaces \cite{yin2025heads}. A contemporaneous 2026 preprint develops a concept-subspace account of structured ICL \cite{tang2026concept}. These findings motivate the explicit prompt-dependent coordinate decoder defined below.

    \paragraph{Relational and graph reasoning.}
        LLMs have been studied on graph problems and relational databases \cite{wang2023graph,wu2025relational}. A contemporaneous 2026 preprint analyzes relational-database ICL through support identifiability and relational label coverage \cite{chen2026openrfm}. Those settings concern database prediction. We instead study logical completion of mathematical partial orders and use order dimension as the representation invariant. Transformer expressivity results establish universality under suitable constructions \cite{qiu2025ask,furuya2025universal}; they do not imply that a finite prompt uniquely identifies its target.

    \paragraph{Order theory.}
        Dushnik--Miller dimension is the minimum number of linear extensions whose intersection is a poset \cite{dushnik1941posets,trotter1992dimension}. The dimension of finite divisibility orders has a developed combinatorial theory \cite{lewis2021divisibility}. We use these classical results rather than claiming them as new; our contribution is their connection to a precisely restricted ICL representation class.
    
\section{Formal Model}    
    Fix a finite known universe $U$. Every target $P=(U,\preceq_P)$ and every hypothesis in the background theory $\B$ has this same ground set.
    
    We represent the labeled relational evidence in a prompt by
    \[
        \D := (\D^+,\D^-), \qquad \D^+,\D^- \subseteq U\times U,
    \]
    where:
    \begin{itemize}
        \item $(x,y)\in\D^+$ is a positive demonstration asserting
        that $x\preceq_P y$; and
        \item $(x,y)\in\D^-$ is a negative demonstration asserting
        that $x\not\preceq_P y$.
    \end{itemize}
    Thus, $\D$ denotes the complete labeled demonstration component of the prompt and does not include the query. We write
    \[
        |\D|:=|\D^+|+|\D^-|
    \]
    for the total number of demonstrations.

    A query $q=(a,b)\in U\times U$ asks whether $a\preceq_P b$. The class $\B$ may additionally encode relation semantics or a completeness assumption. Fixing the common ground set $U$ makes every demonstrated pair and every query well-defined for every hypothesis in $\B$.

    \begin{definition}[Version space and identifiability]
        For a poset $P=(U,\preceq_P)$, let
        \[
        R_P:=\{(x,y)\in U^2:x\preceq_P y\}
        \]
        denote the graph of its order relation. Given a labeled
        demonstration set $\D=(\D^+,\D^-)$, its version space under
        the background theory $\B$ is
        \[
            \begin{aligned}
                \V_{\B}(\D)=\bigl\{P\in\B : {}\;& \D^+\subseteq R_P,\\[-1mm]
                & \D^-\cap R_P=\varnothing\bigr\}.
            \end{aligned}
        \]
        The prompt $\D$ is \emph{satisfiable under $\B$} when
        \[
            \V_{\B}(\D)\neq\varnothing.
        \]
        For a satisfiable prompt, a query $q=(a,b)$ is
        \emph{identified} if
        \[
            \left|\left\{\ind[a\preceq_P b]: P\in\V_{\B}(\D)\right\}\right|=1.
        \]
    \end{definition}

    A binary answer rule is \emph{universally sound} for $(\B,\D,q)$ if it returns the correct answer for every $P\in\V_{\B}(\D)$.
    
    Under \emph{open-world semantics}, $\B$ permits additional comparisons on $U$ beyond those explicitly demonstrated or forced by the partial-order axioms, provided that no negative demonstration is violated. Under \emph{closed-world Hasse semantics}, the displayed finite DAG (Directed acyclic graph) is declared to be the complete Hasse diagram, and $\B$ contains only the poset generated by that
    diagram.

\section{Open-World Identifiability}
    \begin{proposition}[Sound-answer criterion]\label{prop:criterion}
        A universally sound binary answer exists for $(\B,\D,q)$ if and only if $q$ is identified.
    \end{proposition}
    \begin{proof}
        If all consistent posets give value $v$, returning $v$ is sound. If two consistent posets disagree, either binary output is wrong for one of them; randomization cannot guarantee correctness for both.
    \end{proof}
    Supplementary Proposition~S1 gives the fully quantified version, including the randomized-rule case and the nonempty-version-space assumption.
    
    The general criterion is simple but exposes the correct evaluation target. The next theorem gives an exact characterization for arbitrary finite positive and negative comparison prompts under the least restrictive poset background.
    
    For a reflexive relation $R$ and $x\in U$, write
    \[
        \operatorname{Pred}_R(x)=\{u:uRx\},\qquad
        \operatorname{Succ}_R(x)=\{v:xRv\}.
    \]
    
    \begin{theorem}[Open-world completion trichotomy]\label{thm:trichotomy}
        Let $U$ be finite, let $\B$ be the class of all posets on $U$, and let $\D=(\D^+,\D^-)$ be satisfiable. Put $R=\TC(\D^+)$, where $\TC$ denotes reflexive transitive closure. For a query $(a,b)$, exactly one of the following holds:
        \begin{enumerate}
            \item $aRb$, in which case the query is identified true;
            \item $a{\notR}b$ and either $bRa$ or $\bigl(\operatorname{Pred}_R(a)\times\operatorname{Succ}_R(b)\bigr)\cap\D^-\neq\varnothing$,
            in which case the query is identified false;
            \item neither condition holds, in which case the query is unidentifiable.
        \end{enumerate}
    \end{theorem}
    \begin{proof}
        Satisfiability gives a poset containing $\D^+$ and excluding $\D^-$. Therefore $R$ is antisymmetric and $R\cap\D^-=\varnothing$; hence $P^-=(U,R)$ is itself a consistent poset. This disjointness is why, after adding $(a,b)$, it suffices to test only the newly forced predecessor--successor rectangle against $\D^-$. If $aRb$, every poset containing $\D^+$ contains $R$, so every completion makes the query true.
        
        Assume now that $a{\notR}b$. The poset $P^-$ is a consistent completion in which the query is false. It remains to decide whether a true completion exists. Any poset containing $R$ and $a\preceq b$ must also contain every pair $(x,y)$ with $xRa$ and $bRy$, by transitivity. Conversely, the reflexive transitive closure after adding $(a,b)$ is exactly
        \[
            R_{a,b}=R\cup\bigl(\operatorname{Pred}_R(a)\times\operatorname{Succ}_R(b)\bigr).
        \]
        The displayed rectangle is forced by transitivity. For the reverse inclusion, call the right-hand relation $R'$. It contains $R$ and $(a,b)$ and is transitive: composing an $R$-pair with a rectangle pair, or a rectangle pair with an $R$-pair, stays in the rectangle; composing two rectangle pairs does as well. Hence the least reflexive transitive relation containing $R\cup\{(a,b)\}$ is contained in $R'$, proving equality.
        
        If $bRa$, then $R_{a,b}$ contains both $aRb$ and $bRa$ for distinct $a,b$, so no antisymmetric true completion exists. If the displayed cross-product contains a negative demonstration, every true completion violates that demonstration. Thus either condition in item 2 forces false.
        
        Suppose neither obstruction holds. Since $b{\notR}a$, adding $(a,b)$ creates no directed cycle, so $R_{a,b}$ is antisymmetric. By assumption it also avoids every pair in $\D^-$. Therefore $P^+=(U,R_{a,b})$ is a consistent poset in which the query is true, while $P^-$ is a consistent poset in which it is false. The query is unidentifiable. These alternatives are mutually exclusive and exhaustive.
    \end{proof}
    Lemmas~S2--S3 (Appendix) isolate the least-closure and one-edge-closure arguments, and Theorem~S4 (Appendix) restates the trichotomy with a fully modular proof.

    \paragraph{Example.}
        Let $U=\{a,b,c\}$, let $\D^+=\{(a,b)\}$, and let $\D^-=\{(a,c)\}$. The query $a\preceq b$ is forced true. For $b\preceq c$, adding $(b,c)$ forces $(a,c)$ by transitivity, contradicting the negative label, so the query is forced false even though $c\preceq b$ is not known. By contrast, $c\preceq b$ is ambiguous: both the least closure and its extension by $(c,b)$ satisfy the prompt. This illustrates all three outcomes without invoking model behavior.

    \begin{corollary}[Positive-only trichotomy]\label{cor:positive}
        If $\D^-=\varnothing$, a query is identified true when $aRb$, identified false when $a\neq b$ and $bRa$, and otherwise unidentifiable.
    \end{corollary}
    
    Theorem~\ref{thm:trichotomy} is also an algorithm. Compute $R$ once in $O(|U|^3)$ time by a straightforward transitive-closure method. For each query, test $aRb$ and $bRa$, then scan each $(x,y)\in\D^-$ and test $xRa$ and $bRy$; this costs $O(|\D^-|)$ per query after closure preprocessing. With $|U|$-bit predecessor, successor, and negative-adjacency bitsets, the rectangle test costs $O(|U|^2/w)$ word operations per query, where $w$ is the machine word size. A query between two distinct elements absent from all demonstrations falls into the ambiguous case unless background knowledge relates them.
    
    \begin{wrapfigure}[15]{r}{0.5\linewidth}
    \centering
    \vspace{-1.5em}
        \includegraphics[width=0.85\linewidth]{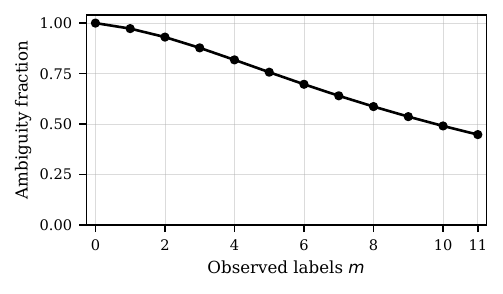}
        \caption{Exact ambiguity fraction over all 219 labeled four-element posets, averaged uniformly over targets, observed $m$-subsets of the 12 nonreflexive ordered pairs, and unobserved queries. Observed labels are target-consistent; the calculation is exhaustive.}
    \label{fig:ambiguity-n4}
    \end{wrapfigure}
    
    \paragraph{Deterministic finite-universe illustration.}
        The trichotomy also permits a model-free prevalence calculation. We exhaustively enumerate all 219 labeled posets on four elements. For each $m\in\{0,\ldots,11\}$, we average uniformly over $m$-subsets of the 12 ordered nonreflexive pairs and over unobserved queries, using target-consistent labels. Figure~\ref{fig:ambiguity-n4} shows that ambiguity decreases with coverage but remains $0.4475$ when 11 of 12 labels are observed. This finite averaging scheme is not a model evaluation or a universal poset distribution; the supplement gives the exact values, and the accompanying code package reproduces them.
    
    The next elementary statement is the standard reachability interpretation of a complete Hasse diagram \cite{trotter1992dimension}; we record it to contrast closed- and open-world semantics.
    \begin{observation}[Closed-world reachability]\label{prop:reach}
    If a finite DAG $H=(U,E)$ is declared to be the complete Hasse diagram, then $a\preceq b$ if and only if $b\in\Reach_H(a)$, including the length-zero path when $a=b$.
    \end{observation}
    \begin{proof}
    A directed path is a chain of cover relations and implies comparability by transitivity. Conversely, let $a\prec b$. If the pair is not a cover, choose $c$ with $a\prec c\prec b$ and refine both intervals. Finiteness makes the refinement terminate in a saturated chain of covers, hence a directed Hasse path. Reflexivity handles $a=b$.
    \end{proof}
    Supplementary Observation~S5 gives the expanded saturated-chain argument and states the required completeness assumption explicitly.
    
    These statements depend on $\B$. If $\B$ fixes the standard numerical order, a query may be identified from background semantics even when demonstrations alone do not identify it. Therefore a benchmark should state whether it measures prompt-only induction, use of pretrained semantic knowledge, or both. Supplementary Proposition~S6 formalizes the monotonicity of identifiability and teaching cost under stronger background knowledge.
    
\section{Prompt Teaching Complexity}
    Identifiability can also be studied as a sample-complexity question. Let $\B_U$ be the class of all posets on the known finite universe $U$. A labeled demonstration set $\D$ \emph{teaches} $P$ under open-world semantics if $\V_{\B_U}(\D)=\{P\}$. Define
    \[
        \tau_{\mathrm{ow}}(P)=\min\{|\D^+|+|\D^-|:\V_{\B_U}(\D)=\{P\}\}.
    \]
    The known ground set is essential: if arbitrary fresh elements are allowed, no finite prompt can isolate an antichain because a new isolated element may always be added. This definition specializes teaching dimension to relation-valued concepts \cite{goldman1995teaching}. Let $\operatorname{Cov}(P)$ be the strict cover relation. For an \emph{ordered} incomparable pair $a\parallel_P b$, define its blocker set
    \[
        B_P(a,b)=\bigl(\operatorname{Pred}_P(a)\times\operatorname{Succ}_P(b)\bigr) \setminus\preceq_P .
    \]
    Let $\beta(P)$ be the minimum size of a set $N\subseteq U^2\setminus\preceq_P$ that intersects $B_P(a,b)$ for every ordered incomparable pair $(a,b)$. The orientations $(a,b)$ and $(b,a)$ are distinct constraints; this convention is essential for the antichain value below.
    
    \begin{theorem}[Exact teaching characterization]\label{thm:teaching-exact}
        For every finite poset $P$,
        \[
        \tau_{\mathrm{ow}}(P)=|\operatorname{Cov}(P)|+\beta(P).
        \]
    \end{theorem}
    \begin{proof}
        Every cover $(x,y)$ must be a positive demonstration. Otherwise deleting only $x\preceq y$ leaves a poset: a transitivity violation would require an intermediate $x\prec z\prec y$, contradicting coverhood. The reduced poset agrees with every other valid label, so the prompt would not teach $P$.
    
        Let $N$ be the negative labels of a teaching set. If $N$ misses $B_P(a,b)$ for some $a\parallel_P b$, then adding $a\preceq b$ and closing transitively gives
        $P\cup(\operatorname{Pred}_P(a)\times\operatorname{Succ}_P(b))$. It is antisymmetric because $b\not\preceq_P a$ and it avoids $N$ by assumption, producing a second consistent poset. Thus $N$ must be a hitting set and $|N|\geq\beta(P)$.
    
        Conversely, label every cover positive and choose a blocker hitting set $N$ of size $\beta(P)$ as negative. The covers generate all of $P$. Any different consistent extension must add some ordered pair $(a,b)\notin P$. The reverse pair cannot lie in $P$ by antisymmetry, so $a\parallel_P b$. Transitivity then forces every pair in $B_P(a,b)$, including a member of $N$, a contradiction. Hence the prompt teaches $P$.
    \end{proof}
    Supplementary Lemma~S7 proves cover necessity separately, while Supplementary Theorem~S8 gives the complete lower- and upper-bound proof with the one-edge extension justified through Supplementary  Lemma~S3.
    
    \begin{corollary}[General bounds and exact extremal values]\label{cor:teaching}
    If $I(P)$ is the number of unordered incomparable pairs, then
    \[
    |\operatorname{Cov}(P)|\leq\tau_{\mathrm{ow}}(P)
    \leq|\operatorname{Cov}(P)|+2I(P).
    \]
    For the $n$-element chain $C_n$ and antichain $A_n$,
    \[
    \tau_{\mathrm{ow}}(C_n)=n-1,\qquad
    \tau_{\mathrm{ow}}(A_n)=n(n-1).
    \]
    Moreover, among all posets on an $n$-element universe,
    \[
    \max_P\tau_{\mathrm{ow}}(P)=n(n-1),
    \]
    and equality is attained uniquely by the antichain.
    \end{corollary}
    \begin{proof}
    Each blocker set contains its defining ordered pair $(a,b)$, so selecting all ordered incomparable pairs is a hitting set of size $2I(P)$. A chain has $n-1$ covers and no incomparable pairs. In an antichain, $B_P(a,b)=\{(a,b)\}$ for every distinct ordered pair, so all $n(n-1)$ negative labels are required.
    
    For the class maximum, the number of unordered comparable pairs is $\binom{n}{2}-I(P)$, and every cover is such a pair. Hence
    \[
        \tau_{\mathrm{ow}}(P)\leq |\operatorname{Cov}(P)|+2I(P)
        \leq \binom{n}{2}+I(P)\leq n(n-1).
    \]
    Equality in the final inequality requires $I(P)=\binom{n}{2}$, so every distinct pair is incomparable and $P$ is the antichain. The antichain attains the bound by the preceding calculation.
    \end{proof}
    Supplementary Corollary~S9 records the edge cases, the class-maximum argument, and the following open- versus closed-world comparison.

    \paragraph{Price of open-world semantics.}
        If a prompt is declared to be the complete Hasse diagram, let $\tau_{\mathrm{cw}}(P)$ be the minimum number of displayed cover edges needed to specify $P$. Every cover is necessary and all covers are sufficient, so
        \[
        \tau_{\mathrm{cw}}(P)=|\operatorname{Cov}(P)|,
        \qquad
        \tau_{\mathrm{ow}}(P)-\tau_{\mathrm{cw}}(P)=\beta(P).
        \]
        Thus $\beta(P)$ is exactly the label cost of open-world semantics.
    
    Theorem~\ref{thm:teaching-exact} turns optimal prompt design into a structured hitting-set problem. Minimum hitting set is NP-hard in general; we do not establish the complexity of the blocker sets induced specifically by posets. The theorem also shows why order dimension alone does not determine teaching cost: a nontrivial antichain has dimension two but quadratic teaching number, whereas a chain has dimension one and linear teaching number. Background theory can shrink both quantities, which is why $\B$ must be reported.

\section{Structural Complexity}
    For a finite prompt--query instance with target poset $P$ and query $q=(a,b)$, we use the profile
    \[
        \begin{aligned}
            \mathsf{Comp}(P,q)=\bigl(&\height(P),\width(P),\dimposet(P),\lambda_P^+(q),\nu_P^-(q),\eta_{\B}(\D,q)\bigr).
        \end{aligned}
    \]
    This is a taxonomy rather than a scalar ordering. Here $\eta_{\B}$ is zero when the query is identified and one otherwise. If $a\preceq_P b$, then $\lambda_P^+(a,b)$ is the shortest Hasse-path length and we set $\nu_P^-(a,b)=\infty$. If $a\not\preceq_P b$, define
    $\nu_P^-(a,b)=|\Reach_{H(P)}(a)|$ and set $\lambda_P^+(a,b)=\infty$; Lemma~\ref{thm:negative} shows that this reachable set is the canonical minimum forward-closed witness. Height and positive path length measure chain depth; width measures incomparability; order dimension measures the number of linear orders needed to realize all comparisons.
    
    For a family-level view, let $L^+(P)=\max_{a\preceq b}\lambda_P^+(a,b)$ and $N^-(P)=\max_{a\not\preceq b}\nu_P^-(a,b)$. Table~\ref{tab:profiles} uses these worst-case query coordinates. For $N=\prod_{i=1}^r p_i^{e_i}$, write $E=\sum_i e_i$ and $W_N=\width(\operatorname{Div}(N))$; for $D_n=([n],\mid)$, let $r(n)=\max\{r:p_1\cdots p_r\leq n\}$.
    
    \begin{table*}[t]
    \centering
    \scriptsize
    \setlength{\tabcolsep}{4pt}
        \begin{tabular}{lcccccc}
            \toprule
            Family & $\height$ & $\width$ & $\dimposet$ & $|\operatorname{Cov}|$ & $L^+$ & $N^-$ \\
            \midrule
            $C_n$ & $n$ & $1$ & $1$ & $n-1$ & $n-1$ & $n-1$ \\
            $B_m$ & $m+1$ & $\binom{m}{\lfloor m/2\rfloor}$ & $m$ & $m2^{m-1}$ & $m$ & $2^{m-1}$ \\
            $\operatorname{Div}(N)$ & $E+1$ & $W_N$ & $r$ & $\sum_i e_i\!\prod_{j\ne i}(e_j+1)$ & $E$ & $\max_i e_i\!\prod_{j\ne i}(e_j+1)$ \\
            $D_n=([n],\mid)$ & $\lfloor\log_2 n\rfloor+1$ & $\geq\lceil n/2\rceil$ & $\geq r(n)$ & $\sum_{p\leq n}\lfloor n/p\rfloor$ & $\lfloor\log_2 n\rfloor$ & $\lfloor n/2\rfloor$ \\
            \bottomrule
        \end{tabular}
        \caption{Structural profiles for $n\geq2$. Entries are exact except the displayed lower bounds for $D_n$; $W_N$ denotes the width of the fixed divisor lattice. The table makes visible that cover count, width, dimension, and positive/negative certificate sizes are independent axes rather than a single notion of difficulty.}
        \label{tab:profiles}
    \end{table*}
    
    For example, a chain has dimension one but can have an arbitrarily long positive certificate, whereas a cover query in a Boolean lattice has certificate length one even though the family has growing width and dimension.
    
    \begin{proposition}[Standard families]\label{prop:families}
        For $n,m,r\geq 1$: (i) a chain $C_n$ has dimension $1$; (ii) the Boolean lattice $B_m=(2^{[m]},\subseteq)$ has dimension $m$; and (iii) if $N=\prod_{i=1}^r p_i^{e_i}$ with distinct primes and $e_i\geq1$, then the poset of positive divisors of $N$ under divisibility has dimension $r$.
    \end{proposition}
    \begin{proof}
        A chain is already a linear order. For $B_1$, the claim is the chain case. For $B_2$, the two membership indicators give the upper bound, while $B_2$ is not a chain, so its dimension is at least two. Assume $m\geq3$. The $m$ membership indicators $\phi_i(A)=\ind[i\in A]$ give an $m$-coordinate realization. By the classical coordinate-order equivalence recorded later as Theorem~\ref{thm:dimension}, this gives an $m$-extension upper bound. For the lower bound, let $a_i=\{i\}$ and $b_i=[m]\setminus\{i\}$. Then $a_i\subseteq b_j$ exactly when $i\neq j$, while each pair $(a_i,b_i)$ is incomparable. In a linear extension, at most one such pair can be reversed: if both $b_i<a_i$ and $b_j<a_j$ held for $i\neq j$, the forced relations $a_i<b_j$ and $a_j<b_i$ would form a cycle. Every incomparable pair must be reversed in some member of a realizer, so at least $m$ extensions are required.
    
        For the divisor poset, map $d\mid N$ to $(v_{p_1}(d),\ldots,v_{p_r}(d))$. Divisibility is coordinatewise comparison, giving dimension at most $r$. The squarefree divisors $\prod_{i\in A}p_i$ induce a copy of $B_r$; dimension is monotone under subposets, so the lower bound is $r$.
    \end{proof}
    Supplementary Theorem~S10 proves the coordinate-order equivalence used for the upper bounds, and Supplementary Proposition~S12 gives the full family arguments, including the $m=1,2$ Boolean-lattice cases.
    
    \begin{corollary}[Growing dimension of finite divisibility]\label{cor:divisibility}
        Let $D_n=([n],\mid)$. If the product of the first $r$ primes is at most $n$, then $\dimposet(D_n)\geq r$. Consequently, the dimensions of $D_n$ are unbounded as $n\to\infty$.
        \end{corollary}
        \begin{proof}
        Let $M=p_1\cdots p_r\leq n$. Every squarefree divisor of $M$ lies in $[n]$, and these divisors induce a copy of $B_r$ under divisibility. Proposition~\ref{prop:families} and monotonicity under subposets give the bound. Sharper asymptotic estimates are known \cite{lewis2021divisibility}.
    \end{proof}
    
    Thus the comparison is family-level: higher order dimension requires more coordinates for exact uniform representation under the decoder defined next, but it does not impose a behavioral difficulty ordering on every individual query or model.

\section{Coordinate-Order Representation Limits}
    \begin{definition}[Prompt-dependent monotone coordinate decoder]\label{def:decoder}
        Let $\mathsf{Prompt}(U)$ be the set of finite labeled prompts over $U$, and fix an integer $s\geq1$. An $s$-coordinate decoder is a fixed map
        \[
            \Phi:\mathsf{Prompt}(U)\longrightarrow (\mathbb{R}^{U})^s,
            \qquad
            \Phi(\D)=(\phi_1^{\D},\ldots,\phi_s^{\D}),
        \]
        together with the conjunctive decision rule
        \[
            \begin{aligned}
                A_{\Phi}(\D;x,y)=1
                \quad\Longleftrightarrow\quad
                &\phi_i^{\D}(x)\leq\phi_i^{\D}(y) \text{ for every }i\in[s].
            \end{aligned}
        \]
        A task family $\mathcal{T}$ is a subset of $\mathsf{Prompt}(U)\times\mathsf{Poset}(U)$ that is functional in its first coordinate: if $(\D,P),(\D,Q)\in\mathcal{T}$, then $P=Q$. Write this unique target as $P_{\D}$. The decoder is \emph{exact on $\mathcal{T}$} if $A_{\Phi}(\D;x,y)=\ind[x\preceq_{P_{\D}}y]$ for every $(\D,P_{\D})\in\mathcal{T}$ and every $(x,y)\in U^2$. The coordinate maps may depend on the prompt, but $\Phi$ is one decoder for the whole family. This is a zero-error exact-representation notion; no claim is made here for approximate or error-tolerant decoding.
    \end{definition}
    For a single poset $P$, an $s$-coordinate order representation means maps $\phi_1,\ldots,\phi_s:U\to\mathbb{R}$ satisfying
    $x\preceq_P y$ if and only if $\phi_i(x)\leq\phi_i(y)$ for every $i\in[s]$.
    The following equivalence is classical; see \citet[Chapter~1]{trotter1992dimension}. We record it with a proof for self-containment because Corollary~\ref{cor:icl} depends on the exact biconditional form.
    \begin{theorem}[Classical coordinate-order equivalence]\label{thm:dimension}
        A finite poset has an $s$-coordinate order representation if and only if $\dimposet(P)\leq s$.
    \end{theorem}
    \begin{proof}
        Suppose $\phi_1,\ldots,\phi_s$ represent $P$. Fix a linear extension $L_0$ of $P$. For each $i$, sort elements by increasing $\phi_i$, breaking ties by $L_0$, to obtain a linear extension $L_i$. If $x\preceq_P y$, every $L_i$ places $x$ before $y$. If $x\not\preceq_P y$, the representation gives some $j$ with $\phi_j(x)>\phi_j(y)$, so $L_j$ places $y$ before $x$. Hence $P=\cap_i L_i$ and $\dimposet(P)\leq s$. Conversely, if $\dimposet(P)=t\leq s$, take a $t$-member realizer and repeat one of its linear extensions until $s$ orders $L_1,\ldots,L_s$ are listed. Let $\phi_i(x)$ be the rank of $x$ in $L_i$. Then the coordinatewise condition is equivalent to membership in every $L_i$, hence to $x\preceq_P y$.
    \end{proof}
    Supplementary Theorem~S10 gives the same proof in expanded form, explicitly checking tie-breaking and both inclusions of the intersection; Supplementary Corollary~S11 gives the task-family version of the capability boundary.

    \begin{corollary}[Exact capability boundary for coordinate decoders]\label{cor:icl}
        Let $\mathcal{T}$ be a task family as in Definition~\ref{def:decoder}. An exact $s$-coordinate decoder exists on $\mathcal{T}$ if and only if $\dimposet(P_{\D})\leq s$ for every $(\D,P_{\D})\in\mathcal{T}$. Consequently:
        \begin{enumerate}
        \item if every target in $\mathcal{T}$ has $\width(P_{\D})\leq s$, then an exact $s$-coordinate decoder exists; and
        \item if some target has $\dimposet(P_{\D})>s$, no exact $s$-coordinate decoder exists on $\mathcal{T}$.
    \end{enumerate}
    \end{corollary}
    \begin{proof}
        If an exact decoder exists, its coordinates for each prompt form an $s$-coordinate representation of the associated target; Theorem~\ref{thm:dimension} gives $\dimposet(P_{\D})\leq s$. Conversely, if every target has dimension at most $s$, choose one $s$-coordinate realization for each target (padding by repeated coordinates when necessary), define $\Phi(\D)$ to return that realization on prompts appearing in $\mathcal{T}$, and extend $\Phi$ arbitrarily elsewhere. This gives one exact decoder on the whole family. For the sufficient width condition, the classical inequality $\dimposet(P)\leq\width(P)$ gives $\dimposet(P_{\D})\leq s$ whenever $\width(P_{\D})\leq s$ \cite{dilworth1950decomposition,trotter1992dimension}. The impossibility statement is the first implication applied to a target of dimension greater than $s$.
    \end{proof}

    This boundary is exact for the decoder class in Definition~\ref{def:decoder}. Its sufficiency direction is representational and may be nonconstructive; it does not assert efficient recovery of the selected coordinates from demonstrations. It also does not show that transformer hidden width, attention-update rank, or the number of arbitrary task vectors bounds order dimension: unrestricted vectors can encode finite combinatorial objects, and unrestricted decoders need not be coordinatewise monotone. Proposition~\ref{prop:families} gives concrete instances: one coordinate for chains, $m$ for $B_m$, and one coordinate per distinct prime for the divisor lattice of a fixed $N$.

\section{Positive and Negative Certificates}
    Let $H(P)$ be the complete Hasse DAG of a finite poset. For $a\preceq b$, define
    \[
        \begin{aligned}
            \lambda_P^+(a,b)=\min\{k :{}& a=x_0\prec x_1\prec\cdots \prec x_k=b\text{ is a cover path}\}.
        \end{aligned}
    \]
    \begin{proposition}[Positive chain-certificate bound]\label{prop:certificate}
        Any positive certificate consisting only of demonstrated cover edges has at least $\lambda_P^+(a,b)$ edges. Consequently, a procedure restricted to composing at most $T$ cover edges cannot certify every positive query with $\lambda_P^+(a,b)>T$.
    \end{proposition}
    \begin{proof}
        Such a certificate is exactly a directed Hasse path, and $\lambda_P^+(a,b)$ is the minimum length of any such path.
    \end{proof}
    Supplementary Proposition~S13 states the certificate model formally and separates proof length from unrestricted reachability computation.

    A false query admits a dual witness. Call $S\subseteq U$ \emph{forward closed} in $H$ if no edge leaves $S$. The following elementary graph fact is recorded for self-containment because it identifies the canonical witness used by $\nu_P^-$. 

    \begin{lemma}[Forward-closed nonreachability witness]\label{thm:negative}
        For vertices $a,b$ of a finite Hasse DAG, there is no directed path from $a$ to $b$ if and only if there exists a forward-closed set $S$ with $a\in S$ and $b\notin S$.
    \end{lemma}
    \begin{proof}
        If such an $S$ exists, a path beginning at $a\in S$ cannot leave $S$, so it cannot reach $b\notin S$. Conversely, if $b$ is not reachable from $a$, take $S=\Reach_H(a)$. It contains $a$ and excludes $b$. If an edge left $S$, its endpoint would also be reachable from $a$, a contradiction. Thus $S$ is forward closed.
    \end{proof}

    \begin{corollary}[Minimality of the canonical negative witness]\label{cor:negative-size}
        If $b$ is not reachable from $a$, then $\Reach_H(a)$ is the unique inclusion-minimal forward-closed set containing $a$ and excluding $b$. It therefore also has minimum cardinality among such sets, and
        $\nu_P^-(a,b)=|\Reach_H(a)|$.
        In particular, for fixed $H$ and $a$, this quantity is independent of which nonreachable vertex $b$ is queried.
    \end{corollary}
    \begin{proof}
        Every forward-closed set containing $a$ must contain every vertex reachable from $a$, while $\Reach_H(a)$ itself is forward closed by Lemma~\ref{thm:negative}.
    \end{proof}
    Supplementary Lemma~S14 and Supplementary Corollary~S15 provide the full equivalence and minimality proof, including induction on path length.

    Positive certificates can be local paths, while the canonical negative witness may expose a large reachable region. This distinction refines the informal claim that false comparability is merely ``absence of a path.'' It still does not yield an unrestricted transformer-depth lower bound: a model may use a global algorithm or an encoded topological summary. Graph-transformer depth--width tradeoffs provide a related architectural perspective \cite{yehudai2025depth}.

\section{Implications, Scope, and Limitations}
    \paragraph{How the proposed framework fits together.}
        Our framework analyzes poset ICL as a sequence of logically distinct questions. First, Proposition~\ref{prop:criterion} asks whether the prompt determines the answer at all: a binary answer is universally sound exactly when every poset in the version space agrees on the query. For the open-world background containing all posets on the fixed universe, Theorem~\ref{thm:trichotomy} turns this general criterion into an explicit classifier with three outcomes---forced true, forced false, or genuinely unknown---while Corollary~\ref{cor:positive} gives the corresponding positive-only case. If the prompt instead declares a complete Hasse diagram, Observation~\ref{prop:reach} reduces comparison to graph reachability. Theorem~\ref{thm:teaching-exact} then changes the question from answering one query to identifying the entire target poset: the optimal open-world prompt consists of all indispensable cover relations together with a minimum set of negative blockers. Corollary~\ref{cor:teaching} gives the resulting bounds and exact extremal values, and the identity
        \[
        \tau_{\mathrm{ow}}(P)-\tau_{\mathrm{cw}}(P)=\beta(P)
        \]
        shows that the blocker term is precisely the additional label cost of open-world semantics. Once the target relation is identified, Table~\ref{tab:profiles} separates its structural and proof-related difficulty into height, width, order dimension, cover count, and positive and negative certificate sizes. Finally, Definition~\ref{def:decoder}, Theorem~\ref{thm:dimension}, and Corollary~\ref{cor:icl} characterize exactly which targets can be represented by the proposed monotone coordinate-decoder class, while Proposition~\ref{prop:certificate}, Lemma~\ref{thm:negative}, and Corollary~\ref{cor:negative-size} describe explicit witnesses for positive comparability and negative nonreachability.

    \paragraph{Examples illustrating the separation of difficulty sources.}
        Consider first the chain $C_n$. It has order dimension one by Proposition~\ref{prop:families}, and its $n-1$ cover edges are sufficient to teach it under both open- and closed-world semantics. Nevertheless, comparing its bottom and top elements requires a positive Hasse certificate of length $n-1$; thus representation is simple even though a particular proof can be long. The antichain $A_n$ demonstrates the opposite teaching behavior. Its complete Hasse diagram contains no cover edges, so it is specified without positive edges under closed-world semantics, but Corollary~\ref{cor:teaching} shows that open-world teaching requires all $n(n-1)$ ordered negative comparisons. The Boolean lattice $B_m$ provides a third contrast: a cover query has a one-edge positive certificate, yet Proposition~\ref{prop:families} gives $\dimposet(B_m)=m$, so Corollary~\ref{cor:icl} rules out exact representation by fewer than $m$ monotone coordinates. Similarly, the divisor lattice of
$N=\prod_{i=1}^{r}p_i^{e_i}$ has dimension $r$, corresponding to the $r$ independent prime-exponent coordinates. These examples show why prompt ambiguity, teaching cost, certificate length, and representation dimension must not be collapsed into a single notion of ``difficulty.''

    \paragraph{Consequences for benchmark design.}
        A poset benchmark should state the universe and background theory, whether demonstrations denote covers or arbitrary comparisons, whether the displayed diagram is complete, and whether \emph{unknown} is an admissible output. Theorem~\ref{thm:trichotomy} can be used as a model-independent logical baseline for mixed positive/negative open-world prompts. Figure~\ref{fig:ambiguity-n4} illustrates why this distinction matters: under the explicitly defined exhaustive averaging scheme over all 219 labeled four-element posets, a substantial fraction of unobserved queries remain ambiguous even at high label coverage. Such queries should not automatically be scored as ordinary binary errors. In addition, prompts using familiar relation names or numerical elements may permit pretrained semantic recall. Abstract relation symbols, randomized element names, or relabeling controls are therefore needed when the objective is to isolate inference from demonstrations rather than prior factual knowledge.

    \paragraph{Scope of the theoretical conclusions.}
        The results separate four possible sources of failure. Non-identifiability is an information limitation established by Proposition~\ref{prop:criterion} and Theorem~\ref{thm:trichotomy}; a large value of $\tau_{\mathrm{ow}}(P)$ is a prompt-budget limitation characterized by Theorem~\ref{thm:teaching-exact}; large values of $\lambda_P^+$ or $\nu_P^-$ are certificate limitations for procedures restricted to the witnesses formalized in Proposition~\ref{prop:certificate} and Lemma~\ref{thm:negative}; and high order dimension is a representation limitation only for the exact monotone coordinate decoders of Definition~\ref{def:decoder}. We therefore do not infer limitations for unrestricted transformers, softmax attention, general task vectors, approximate decoders, or arbitrary graph algorithms. Transformer universality results remain compatible with our theory: they concern what a sufficiently expressive architecture and prompt can compute, whereas our identifiability results concern what a particular finite prompt logically entails.

    \paragraph{Novelty relative to prior and concurrent work.}
        The earlier empirical study of poset ICL \cite{dutta-etal-2025-assessing} motivates the problem through observed model behavior, but it does not provide the version-space trichotomy, the blocker-based teaching characterization, or the exact coordinate-decoder boundary developed here. Information-theoretic analyses of ICL study uncertainty under data-generating distributions \cite{jeon2024information}; in contrast, Proposition~\ref{prop:criterion} and Theorem~\ref{thm:trichotomy} use logical agreement over all relational completions consistent with the prompt. Concurrent relational-database work studies whether support labels identify latent predictive mechanisms \cite{chen2026openrfm}, whereas our results characterize completions of partial orders, their open-world teaching cost, and their Dushnik--Miller dimension. The contribution is therefore not a claim to the first theory of relational learning in general, but a unified framework connecting logical identifiability, optimal prompt teaching, order-theoretic structure, exact coordinate representation, and proof certificates for in-context learning on finite posets.

\section{Conclusion}
    We presented a theory of ICL on partial orders that begins with a basic requirement: the prompt must identify the queried comparison. Finite open-world prompts admit an exact true/false/unknown characterization through reflexive transitive closure, antisymmetry, and negative constraints; complete Hasse prompts reduce to reachability. Open-world teaching cost decomposes into mandatory covers and a blocker-set surcharge, with a tight class maximum at the antichain. For identified structures, the profile table exposes independent dimensions of structural and certificate complexity. Finally, the classical order-dimension equivalence yields an exact capability boundary for the formally defined coordinate-decoder class. The framework supports benchmark design while keeping logical ambiguity, prompt budget, structural complexity, and decoder capacity distinct.

\bibliographystyle{bibstyle} 
\bibliography{references}  

@inproceedings{brown2020language,
  title={Language Models Are Few-Shot Learners},
  author={Brown, Tom B. and Mann, Benjamin and Ryder, Nick and Subbiah, Melanie and Kaplan, Jared and Dhariwal, Prafulla and Neelakantan, Arvind and Shyam, Pranav and Sastry, Girish and Askell, Amanda and Agarwal, Sandhini and Herbert-Voss, Ariel and Krueger, Gretchen and Henighan, Tom and Child, Rewon and Ramesh, Aditya and Ziegler, Daniel M. and Wu, Jeffrey and Winter, Clemens and Hesse, Christopher and Chen, Mark and Sigler, Eric and Litwin, Mateusz and Gray, Scott and Chess, Benjamin and Clark, Jack and Berner, Christopher and McCandlish, Sam and Radford, Alec and Sutskever, Ilya and Amodei, Dario},
  booktitle={Advances in Neural Information Processing Systems},
  volume={33},
  pages={1877--1901},
  year={2020}
}

@inproceedings{garg2022what,
  title={What Can Transformers Learn In-Context? A Case Study of Simple Function Classes},
  author={Garg, Shivam and Tsipras, Dimitris and Liang, Percy S. and Valiant, Gregory},
  booktitle={Advances in Neural Information Processing Systems},
  volume={35},
  pages={30583--30598},
  year={2022}
}

@inproceedings{akyurek2023what,
  title={What Learning Algorithm Is In-Context Learning? Investigations with Linear Models},
  author={Aky{\"u}rek, Ekin and Schuurmans, Dale and Andreas, Jacob and Ma, Tengyu and Zhou, Denny},
  booktitle={International Conference on Learning Representations},
  year={2023}
}

@inproceedings{dai2023why,
  title={Why Can {GPT} Learn In-Context? Language Models Secretly Perform Gradient Descent as Meta-Optimizers},
  author={Dai, Damai and Sun, Yutao and Dong, Li and Hao, Yaru and Ma, Shuming and Sui, Zhifang and Wei, Furu},
  booktitle={Findings of the Association for Computational Linguistics: ACL 2023},
  pages={4005--4019},
  publisher={Association for Computational Linguistics},
  year={2023}
}

@inproceedings{guo2024how,
  title={How Do Transformers Learn In-Context Beyond Simple Functions? A Case Study on Learning with Representations},
  author={Guo, Tianyu and Hu, Wei and Mei, Song and Wang, Huan and Xiong, Caiming and Savarese, Silvio and Bai, Yu},
  booktitle={International Conference on Learning Representations},
  year={2024}
}

@inproceedings{bhattamishra2024understanding,
  title={Understanding In-Context Learning in Transformers and {LLM}s by Learning to Learn Discrete Functions},
  author={Bhattamishra, Satwik and Patel, Arkil and Blunsom, Phil and Kanade, Varun},
  booktitle={International Conference on Learning Representations},
  year={2024}
}

@inproceedings{jeon2024information,
  title={An Information-Theoretic Analysis of In-Context Learning},
  author={Jeon, Hong Jun and Lee, Jason D. and Lei, Qi and Van Roy, Benjamin},
  booktitle={Proceedings of the 41st International Conference on Machine Learning},
  series={Proceedings of Machine Learning Research},
  volume={235},
  pages={21522--21554},
  publisher={PMLR},
  year={2024}
}

@inproceedings{hendel2023task,
  title={In-Context Learning Creates Task Vectors},
  author={Hendel, Roee and Geva, Mor and Globerson, Amir},
  booktitle={Findings of the Association for Computational Linguistics: EMNLP 2023},
  pages={9318--9333},
  publisher={Association for Computational Linguistics},
  year={2023}
}

@inproceedings{yin2025heads,
  title={Which Attention Heads Matter for In-Context Learning?},
  author={Yin, Kayo and Steinhardt, Jacob},
  booktitle={Proceedings of the 42nd International Conference on Machine Learning},
  series={Proceedings of Machine Learning Research},
  volume={267},
  pages={72428--72461},
  publisher={PMLR},
  year={2025}
}

@article{tang2026concept,
  title={In-Context Learning Operates as Concept Subspace Learning},
  author={Tang, Wei and Jiang, Xinyan and Karray, Fakhri and Hu, Lijie},
  journal={arXiv preprint arXiv:2605.18830},
  year={2026}
}

@article{chen2026openrfm,
  title={{OpenRFM}: Dissecting Relational In-Context Learning},
  author={Chen, Zhikai and Yin, Junyu and Gu, Jialiang and Xiong, Siheng and Liu, Xiaoze and Zhang, Ruowang and Zhou, Keren and Guo, Kai},
  journal={arXiv preprint arXiv:2606.04320},
  year={2026}
}

@inproceedings{wu2025relational,
  title={Large Language Models Are Good Relational Learners},
  author={Wu, Fang and Dwivedi, Vijay Prakash and Leskovec, Jure},
  booktitle={Proceedings of the 63rd Annual Meeting of the Association for Computational Linguistics (Volume 1: Long Papers)},
  pages={7835--7854},
  address={Vienna, Austria},
  publisher={Association for Computational Linguistics},
  doi={10.18653/v1/2025.acl-long.386},
  year={2025}
}

@inproceedings{qiu2025ask,
  title={Ask, and It Shall Be Given: On the Turing Completeness of Prompting},
  author={Qiu, Ruizhong and Xu, Zhe and Bao, Wenxuan and Tong, Hanghang},
  booktitle={International Conference on Learning Representations},
  year={2025}
}

@inproceedings{furuya2025universal,
  title={Transformers Are Universal In-Context Learners},
  author={Furuya, Takashi and de Hoop, Maarten V. and Peyr{\'e}, Gabriel},
  booktitle={International Conference on Learning Representations},
  year={2025}
}

@inproceedings{wang2023graph,
  title={Can Language Models Solve Graph Problems in Natural Language?},
  author={Wang, Heng and Feng, Shangbin and He, Tianxing and Tan, Zhaoxuan and Han, Xiaochuang and Tsvetkov, Yulia},
  booktitle={Advances in Neural Information Processing Systems},
  volume={36},
  pages={30840--30861},
  year={2023}
}

@inproceedings{yehudai2025depth,
  title={Depth-Width Tradeoffs for Transformers on Graph Tasks},
  author={Yehudai, Gilad and Sanford, Clayton and Bechler-Speicher, Maya and Fischer, Orr and Gilad-Bachrach, Ran and Globerson, Amir},
  booktitle={Advances in Neural Information Processing Systems},
  volume={38},
  pages={22949--22981},
  year={2025}
}

@article{dushnik1941posets,
  title={Partially Ordered Sets},
  author={Dushnik, Ben and Miller, E. W.},
  journal={American Journal of Mathematics},
  volume={63},
  number={3},
  pages={600--610},
  year={1941}
}

@book{trotter1992dimension,
  title={Combinatorics and Partially Ordered Sets: Dimension Theory},
  author={Trotter, William T.},
  publisher={Johns Hopkins University Press},
  address={Baltimore, Maryland},
  year={1992}
}

@article{dilworth1950decomposition,
  title={A Decomposition Theorem for Partially Ordered Sets},
  author={Dilworth, Robert P.},
  journal={Annals of Mathematics},
  volume={51},
  number={1},
  pages={161--166},
  year={1950},
  doi={10.2307/1969503}
}

@article{lewis2021divisibility,
  title={The Order Dimension of Divisibility},
  author={Lewis, David and Souza, Victor},
  journal={Journal of Combinatorial Theory, Series A},
  volume={179},
  pages={105391},
  year={2021}
}

@inproceedings{mitchell1977version,
  title={Version Spaces: A Candidate Elimination Approach to Rule Learning},
  author={Mitchell, Tom M.},
  booktitle={Proceedings of the Fifth International Joint Conference on Artificial Intelligence},
  pages={305--310},
  year={1977}
}

@incollection{reiter1978closed,
  title={On closed world data bases},
  author={Reiter, Raymond},
  booktitle={Readings in Artificial Intelligence and Databases},
  pages={248--258},
  year={1989},
  publisher={Elsevier}
}

@inproceedings{dutta-etal-2025-assessing,
    title = "Assessing the Limits of In-Context Learning beyond Functions using Partially Ordered Relation",
    author = "Dutta, Debanjan  and
      Ansari, Faizanuddin  and
      Das, Swagatam",
    booktitle = "Proceedings of the 14th International Joint Conference on Natural Language Processing and the 4th Conference of the Asia-Pacific Chapter of the Association for Computational Linguistics",
    year = "2025",
    publisher = "The Asian Federation of Natural Language Processing and The Association for Computational Linguistics",
    url = "https://aclanthology.org/2025.ijcnlp-long.50/",
    doi = "10.18653/v1/2025.ijcnlp-long.50",
    pages = "900--918",
    ISBN = "979-8-89176-298-5"
}

@article{goldman1995teaching,
  title={On the Complexity of Teaching},
  author={Goldman, Sally A. and Kearns, Michael J.},
  journal={Journal of Computer and System Sciences},
  volume={50},
  number={1},
  pages={20--31},
  year={1995}
}
\clearpage
\setcounter{section}{0}
\setcounter{theorem}{0}
\renewcommand{\thetheorem}{S\arabic{theorem}}

\title{Supplementary Material for ``Identifiability and Order-Dimension Limits of In-Context Learning on Partial Orders''}
\author{}
\makeappendixtitle

\begin{abstract}
    This appendix supplies expanded proofs, edge-case checks, and the exact enumeration underlying the deterministic ambiguity illustration in the main paper. Results use the prefix S, and the opening map links each main-paper result to the corresponding statement here. We also give the full class-maximum teaching argument, the open- versus closed-world teaching comparison, the exact coordinate-decoder capability boundary, and details behind the structural profile table.
\end{abstract}

    \section{Cross-Document Reference Map}
        The theorem titles below are copied exactly from the main paper. Each item gives the expanded proof in this document.
        \begin{itemize}
            \item \textbf{Sound-answer criterion:} Proposition S1.
            \item \textbf{Open-world completion trichotomy and positive-only trichotomy:} Lemmas S2--S3 and Theorem S4.
            \item \textbf{Deterministic ambiguity illustration in Figure 1 of the main paper:} the section ``Exact Finite-Universe Ambiguity Enumeration.''
            \item \textbf{Closed-world reachability:} Observation S5.
            \item \textbf{Dependence on background theory:} Proposition S6.
            \item \textbf{Exact teaching characterization:} Lemma S7 and Theorem S8.
            \item \textbf{Teaching bounds, class maximum, and open/closed-world cost:} Corollary S9.
            \item \textbf{Classical coordinate-order equivalence:} Theorem S10.
            \item \textbf{Exact coordinate-decoder capability boundary:} Corollary S11.
            \item \textbf{Standard families, profile-table details, and finite divisibility:} Proposition S12 and the following profile derivations.
            \item \textbf{Positive chain-certificate bound:} Proposition S13.
            \item \textbf{Forward-closed witness and canonical negative witness:} Lemma S14 and Corollary S15.
        \end{itemize}

    \section{Semantic Setup and Sound Answers}
        Let $U$ be a finite known universe. A prompt is $\D=(\D^+,\D^-)$, where $\D^+\subseteq U^2$ contains positive comparisons and $\D^-\subseteq U^2$ contains negative comparisons. A background theory $\B$ is a class of posets on $U$. Its version space is
        \[
            \V_{\B}(\D)=\{P\in\B:\D^+\subseteq\preceq_P,\ \D^-\cap\preceq_P=\varnothing\}.
        \]
        We call $\D$ \emph{satisfiable under $\B$} when $\V_{\B}(\D)\neq\varnothing$, and all statements below assume satisfiability. A binary answer rule is universally sound for $q=(a,b)$ if its output equals $\ind[a\preceq_P b]$ for every $P\in\V_{\B}(\D)$.
    
        \begin{proposition}[Sound-answer criterion]\label{supp:criterion}
            A universally sound binary answer exists for $(\B,\D,q)$ if and only if all posets in $\V_{\B}(\D)$ assign the same truth value to $q$.
        \end{proposition}
        \begin{proof}
            If all consistent posets assign the common value $v\in\{0,1\}$, the rule that returns $v$ is correct for every member of the version space. Conversely, if the query is not identified, there are $P_0,P_1\in\V_{\B}(\D)$ with values zero and one, respectively. Every deterministic binary answer is therefore wrong on one of them. A randomized rule cannot guarantee correctness on both either: each realized output in its support is zero or one and is wrong on one of $P_0,P_1$. Thus universal, sure correctness is possible exactly under agreement of the version space. This statement concerns logical identifiability, not average-case accuracy under a distribution over targets.
        \end{proof}
    
    \section{Expanded Open-World Completion Proof}
        For the least-restrictive open-world background, let $\B$ contain all posets on $U$ and put $R=\TC(\D^+)$, where $\TC$ is reflexive transitive closure. Define
        \[
            \operatorname{Pred}_R(x)=\{u:uRx\},\qquad \operatorname{Succ}_R(x)=\{v:xRv\}.
        \]
        
        \begin{lemma}[Consistency of the least positive closure]\label{supp:least}
            If $\D$ is satisfiable, then $R$ is a partial order, $R\cap\D^-=\varnothing$, and $R$ is contained in every poset consistent with $\D$.
        \end{lemma}
        \begin{proof}
            Reflexivity and transitivity hold by construction. Let $P$ be any consistent poset. Every $R$-comparison is witnessed by a path of positive demonstrations, so transitivity of $P$ gives $R\subseteq\preceq_P$. If distinct $x,y$ satisfied both $xRy$ and $yRx$, then $P$ would contain both comparisons, contradicting antisymmetry. Thus $R$ is antisymmetric. If $(x,y)\in R\cap\D^-$, every consistent $P$ would both contain and exclude $(x,y)$, contradicting satisfiability. The same path argument proves containment in every consistent poset.
        \end{proof}
        
        \begin{lemma}[One-edge closure formula]\label{supp:rectangle}
            Let $R$ be reflexive and transitive, and suppose $a{\notR}b$. Then
            \[
                \TC(R\cup\{(a,b)\})
                =R\cup\bigl(\operatorname{Pred}_R(a)\times
                         \operatorname{Succ}_R(b)\bigr).
            \]
            If $R$ is antisymmetric, the displayed closure is antisymmetric if and only if $b{\notR}a$.
        \end{lemma}
        \begin{proof}
            Call the relation on the right $R'$. It contains $R$ and $(a,b)$ because $aRa$ and $bRb$. If $xRa$ and $bRy$, then the path $xRa$, $a\to b$, $bRy$ shows that $(x,y)$ belongs to the left-hand closure. Hence $R'\subseteq\TC(R\cup\{(a,b)\})$.
            
            For the reverse inclusion, it is enough to prove that $R'$ is reflexive and transitive. Reflexivity comes from $R$. Consider $(x,y),(y,z)\in R'$. If both lie in $R$, then $xRz$. If the first lies in $R$ and the second in the rectangle, then $xRy$ and $yRa$ imply $xRa$, while $bRz$ holds, so $(x,z)$ lies in the rectangle. If the first lies in the rectangle and the second in $R$, then $xRa$ and $bRyRz$, again giving a rectangle pair. If both lie in the rectangle, the first gives $xRa$ and the second gives $bRz$, which is sufficient. Thus $R'$ is transitive, and minimality of transitive closure gives the reverse inclusion.
            
            If $bRa$, then adding $aRb$ creates a two-way comparison between distinct elements, so antisymmetry fails. Conversely, suppose $b{\notR}a$ and consider the directed graph whose nonreflexive edges are those of $R$ together with the single new edge $a\to b$. Any directed cycle on distinct vertices that was not already present in $R$ must use $a\to b$; removing that edge gives an $R$-path from $b$ to $a$, contradicting $b{\notR}a$. The graph is therefore acyclic on distinct vertices, and its transitive closure is antisymmetric.
        \end{proof}
        
        \begin{theorem}[Open-world completion trichotomy]\label{supp:trichotomy}
            Let $U$ be finite, let $\B$ be the class of all posets on $U$, and let $\D=(\D^+,\D^-)$ be satisfiable. Put $R=\TC(\D^+)$. For a query $(a,b)$, exactly one of the following holds:
            \begin{enumerate}
                \item $aRb$, and the query is identified true;
                \item $a{\notR}b$ and either $bRa$ or $\bigl(\operatorname{Pred}_R(a)\times\operatorname{Succ}_R(b)\bigr)\cap\D^-\neq\varnothing$, and the query is identified false;
                \item neither condition holds, and the query is unidentifiable.
            \end{enumerate}
        \end{theorem}
        \begin{proof}
            By Lemma~\ref{supp:least}, $P^-=(U,R)$ is a consistent poset. If $aRb$, every consistent poset contains $R$, so every one makes the query true.
            
            Assume $a{\notR}b$. Then $P^-$ is a consistent false completion. Any true completion must contain $R$, the comparison $a\preceq b$, and therefore every pair in
            $\operatorname{Pred}_R(a)\times\operatorname{Succ}_R(b)$ by transitivity. Lemma~\ref{supp:rectangle} shows that adding precisely this rectangle is already the least reflexive transitive true extension.
            
            If $bRa$, Lemma~\ref{supp:rectangle} shows that every true extension violates antisymmetry. If the rectangle intersects $\D^-$, every true extension violates a negative demonstration. Either obstruction therefore makes the query forced false.
            
            Finally suppose that neither obstruction occurs. Lemma~\ref{supp:rectangle} shows that
            \[
            P^+=\left(U,R\cup
            \bigl(\operatorname{Pred}_R(a)\times\operatorname{Succ}_R(b)\bigr)\right)
            \]
            is reflexive, transitive, and antisymmetric. Lemma~\ref{supp:least} says $R$ avoids $\D^-$, and the assumed empty intersection says the added pairs avoid $\D^-$ as well. Thus $P^+$ is a consistent true completion, while $P^-$ is a consistent false completion. The query is unidentifiable. The three cases are mutually exclusive and cover all possibilities.
        \end{proof}
        
        If $\D^-=\varnothing$, Theorem~\ref{supp:trichotomy} reduces to: true if $aRb$, false if $a\neq b$ and $bRa$, and unknown otherwise. Algorithmically, Floyd--Warshall gives $O(|U|^3)$ closure preprocessing. Afterward, a query costs $O(|\D^-|)$ naively by scanning negative pairs and testing membership in the predecessor--successor rectangle. With $|U|$-bit reachability and negative-adjacency bitsets, the rectangle test costs $O(|U|^2/w)$ word operations per query, where $w$ is the machine word size. This is a logical oracle, not a learned predictor.
        
    \section{Exact Finite-Universe Ambiguity Enumeration}
        Figure~1 of the main paper is an exhaustive consequence of the version-space definition, not a language-model experiment. Let $\mathcal{P}_4$ be the set of all 219 labeled posets on $U=\{1,2,3,4\}$, and let
        \[
            Q_4=\{(i,j)\in U^2:i\neq j\},\qquad |Q_4|=12.
        \]
        The four reflexive pairs are omitted because every poset labels them true, so observing them would add no information. For a target $P\in\mathcal{P}_4$ and an observed-pair set $S\subseteq Q_4$, define the complete target-consistent prompt on $S$ by
        \[
            \D_S(P)^+=S\cap\preceq_P,
            \qquad
            \D_S(P)^-=S\setminus\preceq_P.
        \]
        Its version space is
        \[
            \V(P,S)=\{P'\in\mathcal{P}_4:
            \preceq_{P'}\cap S=\preceq_P\cap S\}.
        \]
        For an unobserved query $q\in Q_4\setminus S$, write
        \[
            A(P,S,q)=
            \ind\!\left[
            \{\ind[q\in\preceq_{P'}]:P'\in\V(P,S)\}=\{0,1\}
            \right].
        \]
        Thus $A(P,S,q)=1$ exactly when the prompt leaves $q$ ambiguous. For prompt size $m\in\{0,\ldots,11\}$, the plotted quantity is
        \[
            \alpha_m=
            \frac{1}{219\binom{12}{m}(12-m)}
            \sum_{P\in\mathcal{P}_4}
            \sum_{\substack{S\subseteq Q_4\\|S|=m}}
            \sum_{q\in Q_4\setminus S}
            A(P,S,q).
        \]
        Every target, prompt subset, and unobserved query therefore receives equal weight. There is no random-poset generator, fitted parameter, or Monte Carlo error.
        
        \begin{center}
        \begin{minipage}{0.98\columnwidth}
        \centering
        \small
        \setlength{\tabcolsep}{5.2pt}
        \begin{tabular}{rcccccc}
            \toprule
            $m$ & 0 & 1 & 2 & 3 & 4 & 5 \\
            $\alpha_m$ & 1.0000 & .9726 & .9306 & .8772 & .8177 & .7566 \\
            \midrule
            $m$ & 6 & 7 & 8 & 9 & 10 & 11 \\
            $\alpha_m$ & .6968 & .6398 & .5863 & .5365 & .4902 & .4475 \\
        \bottomrule
        \end{tabular}
        \captionof{table}{Exact values plotted in Figure 1 of the main paper.}
        \label{tab:supp-ambiguity-values}
        \end{minipage}
        \end{center}
        
        The computation groups posets by their restrictions to the observed set $S$. Within one such version-space cell, an unobserved query is ambiguous precisely when its truth bit is not constant across the cell. Enumerating all $2^{12}$ observed-pair masks and all 219 posets therefore yields the exact counts without repeatedly solving a completion problem. The supplied script writes both the CSV values and the publication figure. These values illustrate the prevalence of the unknown outcome for one fully specified finite averaging scheme; they are not asserted to approximate a natural distribution over larger posets.
        
        \paragraph{Implementation details.}
            The bundled script uses Python 3.13.5 and integer bit masks for the exact enumeration; Matplotlib 3.10.8 is used only to render the figure. The supplied run used a CPU-only AMD EPYC 9V74 environment with nine logical cores and 5.9 GiB of visible memory. It completed in approximately 1.6 seconds with a peak resident set size of 164 MiB. The exact integer counts are independent of hardware, and no random seed is required.
        
            The following is the standard reachability interpretation of a complete Hasse diagram; it is recorded for self-containment.
        \begin{observation}[Closed-world reachability]\label{supp:closed}
            If a finite DAG $H=(U,E)$ is declared to be the complete Hasse diagram of a poset, then $a\preceq b$ if and only if $b\in\Reach_H(a)$, including the length-zero path for $a=b$.
        \end{observation}
        \begin{proof}
            A directed Hasse path is a chain of cover relations, so transitivity implies comparability. Conversely, suppose $a\prec b$. If $(a,b)$ is a cover, it is an edge. Otherwise choose $c$ with $a\prec c\prec b$ and refine the two intervals. Each refinement strictly decreases the size of the corresponding finite interval, so the process terminates in a saturated chain $a=x_0\prec x_1\prec\cdots\prec x_k=b$ whose adjacent pairs are covers. This is a directed path in $H$. Reflexivity handles $a=b$.
        \end{proof}
        
        \begin{proposition}[Monotonicity under stronger background knowledge]\label{supp:background}
            Let $\B_1\subseteq\B_2$. Then $\V_{\B_1}(\D)\subseteq\V_{\B_2}(\D)$. Consequently, a query identified under $\B_2$ is identified with the same value under $\B_1$. For a fixed target $P\in\B_1$, the minimum teaching-set size relative to $\B_1$ is no greater than the minimum relative to $\B_2$.
        \end{proposition}
        \begin{proof}
            The version-space inclusion follows directly because every hypothesis admitted by $\B_1$ is admitted by $\B_2$. Agreement on the larger version space implies agreement on its subset. Any label set isolating $P$ among $\B_2$ also isolates it among the smaller class $\B_1$, so restricting the background cannot increase the minimum teaching size.
        \end{proof}
    
    \section{Expanded Teaching-Number Proof}
        Let $\B_U$ be the class of all posets on a known finite universe $U$. A prompt teaches $P$ under open-world semantics if $\V_{\B_U}(\D)=\{P\}$. Define
        \[
            \tau_{\mathrm{ow}}(P)=\min\{|\D^+|+|\D^-|:\V_{\B_U}(\D)=\{P\}\}.
        \]
        The fixed universe is essential: if fresh elements are permitted, a finite prompt cannot isolate an antichain because another isolated element can be added. Let $\operatorname{Cov}(P)$ denote the strict covers. For an \emph{ordered} incomparable pair $a\parallel_P b$, define
        \[
            B_P(a,b)= \bigl(\operatorname{Pred}_P(a)\times\operatorname{Succ}_P(b)\bigr) \setminus\preceq_P.
        \]
        Let $\beta(P)$ be the minimum size of a set $N\subseteq U^2\setminus\preceq_P$ that intersects every $B_P(a,b)$. The ordered constraints $(a,b)$ and $(b,a)$ are counted separately.
        
        \begin{lemma}[Cover necessity]\label{supp:cover}
            Every strict cover $(x,y)$ of $P$ must appear as a positive label in every teaching set for $P$.
        \end{lemma}
        \begin{proof}
            Suppose the positive pair $(x,y)$ is omitted. Delete only this pair from $P$ and call the resulting relation $Q$. Reflexivity and antisymmetry are unchanged. If transitivity failed, there would be $u\preceq_Q v\preceq_Q w$ whose required pair is the only deleted pair, so $u=x$ and $w=y$. The intermediate vertex cannot be $x$ or $y$, because one of the two premises would then itself be the deleted pair. Hence $x\prec_P v\prec_P y$, contradicting that $(x,y)$ is a cover. Thus $Q$ is a poset. Every other positive label remains true, and every negative label remains false because $Q\subset P$. The prompt would therefore be consistent with $Q\neq P$, contradicting teaching.
        \end{proof}
        
        \begin{theorem}[Exact teaching characterization]\label{supp:teaching}
            For every finite poset $P$,
            \[
                \tau_{\mathrm{ow}}(P)=|\operatorname{Cov}(P)|+\beta(P).
            \]
        \end{theorem}
        \begin{proof}
            For the lower bound, Lemma~\ref{supp:cover} requires all covers as positive labels. Let $N$ be the negative labels of any teaching set. Suppose $N$ misses $B_P(a,b)$ for an ordered incomparable pair. Apply Lemma~\ref{supp:rectangle} with $R=\preceq_P$. Since $a\parallel_P b$, in particular $b\not\preceq_P a$, so the extension
            \[
                P_{a,b}=P\cup \bigl(\operatorname{Pred}_P(a)\times\operatorname{Succ}_P(b)\bigr)
            \]
            is a poset. Its newly added pairs are exactly contained in $B_P(a,b)$, and by assumption none is in $N$. It contains all positive labels and violates no negative label, contradicting uniqueness. Hence $N$ hits every blocker set, so $|N|\geq\beta(P)$ and $\tau_{\mathrm{ow}}(P)\geq|\operatorname{Cov}(P)|+\beta(P)$.
        
            For the upper bound, label every cover positive and choose a minimum blocker hitting set $N$ as negative. In a finite poset, every strict comparison lies on a saturated chain of covers; therefore every consistent poset $Q$ contains all of $P$. If $Q\neq P$, choose $(a,b)\in Q\setminus P$. The reverse pair cannot belong to $P$, because then $Q$ would contain both directions between distinct elements and violate antisymmetry. Thus $a\parallel_P b$. Since $Q$ contains $P$ and $a\preceq_Q b$, transitivity forces every pair in $\operatorname{Pred}_P(a)\times\operatorname{Succ}_P(b)$, including every member of $B_P(a,b)$. The hitting set contains some negative label in this blocker set, contradicting consistency of $Q$. Hence $Q=P$, proving the matching upper bound.
        \end{proof}
        
        \begin{corollary}[General bounds, exact extrema, and semantic cost]\label{supp:teaching-values}
            If $I(P)$ is the number of unordered incomparable pairs, then
            \[
                |\operatorname{Cov}(P)|\leq\tau_{\mathrm{ow}}(P)\leq|\operatorname{Cov}(P)|+2I(P).
            \]
            For the $n$-element chain $C_n$ and antichain $A_n$,
            \[
                \tau_{\mathrm{ow}}(C_n)=n-1,\qquad \tau_{\mathrm{ow}}(A_n)=n(n-1).
            \]
            Moreover,
            \[
                \max_{P\text{ on }n\text{ elements}}\tau_{\mathrm{ow}}(P)=n(n-1),
            \]
            with equality uniquely for $A_n$. If $\tau_{\mathrm{cw}}(P)$ counts the cover edges needed when a prompt is declared to be the complete Hasse diagram, then
            \[
                \tau_{\mathrm{cw}}(P)=|\operatorname{Cov}(P)|, \qquad \tau_{\mathrm{ow}}(P)-\tau_{\mathrm{cw}}(P)=\beta(P).
            \]
        \end{corollary}
        \begin{proof}
            For every ordered incomparable pair $(a,b)$, reflexivity gives $(a,b)\in B_P(a,b)$. Selecting all ordered incomparable pairs is therefore a hitting set of size $2I(P)$, proving the upper bound; the cover term gives the lower bound. A chain has $n-1$ covers and no incomparable pairs, including the edge case $n=1$. In an antichain, $B_P(a,b)=\{(a,b)\}$ for every distinct ordered pair, so all $n(n-1)$ negative labels are necessary and there are no covers.
        
            For any $n$-element poset, the number of unordered comparable pairs is $\binom{n}{2}-I(P)$ and covers form a subset of them. Therefore
            \[
                \tau_{\mathrm{ow}}(P) \leq |\operatorname{Cov}(P)|+2I(P) \leq \binom{n}{2}+I(P) \leq n(n-1).
            \]
            If equality holds, then the last inequality forces $I(P)=\binom{n}{2}$, so every distinct pair is incomparable and $P=A_n$. The antichain calculation shows attainability.
        
            Under complete-Hasse semantics, displaying all covers is sufficient because their reflexive transitive closure is $P$. Every cover is necessary: omitting it changes the declared complete Hasse diagram and therefore the generated poset. Thus $\tau_{\mathrm{cw}}(P)=|\operatorname{Cov}(P)|$. Theorem~\ref{supp:teaching} then gives the difference $\beta(P)$.
        \end{proof}
        
        The quantity $\beta(P)$ is a minimum hitting-set instance. Minimum hitting set is NP-hard in general, but we do not determine the complexity of the structured blocker-set instances induced by posets.
        
        \section{Order Dimension and the Conditional Decoder Bound}
        
        Order dimension is classical \cite{dushnik1941posets,trotter1992dimension}. The following equivalence is classical as well; see \citet[Chapter~1]{trotter1992dimension}. We record the exact form because the prompt-dependent decoder corollary uses both directions.
        
        \begin{theorem}[Coordinate-order equivalence]\label{supp:coordinate}
        For a nonempty finite poset $P$ and an integer $s\geq1$, the following are equivalent:
        \begin{enumerate}
        \item there are maps $\phi_i:U\to\mathbb{R}$, $i\in[s]$, such that
        \[
        x\preceq_P y\quad\Longleftrightarrow\quad
        \phi_i(x)\leq\phi_i(y)\quad\text{for every }i;
        \]
        \item $\dimposet(P)\leq s$.
        \end{enumerate}
        \end{theorem}
        \begin{proof}
        Assume the coordinate representation. Fix a linear extension $L_0$ of $P$. For each $i$, order the elements by increasing $\phi_i$ and break ties according to $L_0$; call the resulting total order $L_i$. If $x\preceq_P y$, then $\phi_i(x)\leq\phi_i(y)$ for every $i$. A strict inequality puts $x$ before $y$ directly, and an equality is resolved in the same direction by $L_0$. Thus every $L_i$ extends $P$.
        
        If $x\not\preceq_P y$, the biconditional in the coordinate representation implies that some $j$ satisfies $\phi_j(x)>\phi_j(y)$. Hence $L_j$ places $y$ before $x$, so $(x,y)$ is absent from the intersection of the $L_i$. We have shown both inclusions
        $P\subseteq\bigcap_i L_i$ and $\bigcap_i L_i\subseteq P$. Therefore the $L_i$ form a realizer and $\dimposet(P)\leq s$.
        
        Conversely, if $\dimposet(P)=t\leq s$, take a $t$-member realizer and repeat one member until $s$ orders $L_1,\ldots,L_s$ are listed. Let $\phi_i(x)$ be the rank of $x$ in $L_i$. Then $\phi_i(x)\leq\phi_i(y)$ for all $i$ exactly when $x$ precedes $y$ in every realizer order, which is equivalent to $x\preceq_P y$.
        \end{proof}
        
        Let $\mathsf{Prompt}(U)$ be the prompts over $U$, and fix an integer $s\geq1$. An $s$-coordinate decoder is a fixed map $\Phi:\mathsf{Prompt}(U)\to(\mathbb{R}^U)^s$ with the conjunctive coordinatewise decision rule. A task family $\mathcal{T}\subseteq\mathsf{Prompt}(U)\times\mathsf{Poset}(U)$ is functional in its first coordinate: $(\D,P),(\D,Q)\in\mathcal{T}$ implies $P=Q$.
        
        \begin{corollary}[Exact coordinate-decoder capability boundary]\label{supp:decoder-boundary}
        An exact $s$-coordinate decoder exists on $\mathcal{T}$ if and only if every target $P_{\D}$ in the family satisfies $\dimposet(P_{\D})\leq s$. In particular, width at most $s$ for every target is sufficient, whereas one target of dimension greater than $s$ makes exact decoding impossible.
        \end{corollary}
        \begin{proof}
        If a decoder is exact, its coordinates at prompt $\D$ form an $s$-coordinate representation of $P_{\D}$, so Theorem~\ref{supp:coordinate} gives $\dimposet(P_{\D})\leq s$. Conversely, if every target has dimension at most $s$, choose an $s$-coordinate realization for each target, padding by repeated coordinates when needed. Define the fixed map $\Phi$ to return that realization on prompts occurring in $\mathcal{T}$ and extend it arbitrarily to all other prompts. The width statement follows from the classical inequality $\dimposet(P)\leq\width(P)$ \cite{dilworth1950decomposition,trotter1992dimension}.
        \end{proof}
        
        This is an exact, zero-error representation boundary for the specified decoder form. The sufficiency argument is not an efficient learning algorithm, and approximate decoders and unrestricted neural representations are outside the statement.
        
        \begin{proposition}[Standard poset families and finite divisibility]\label{supp:families}
        For $n,m,r\geq1$:
        \begin{enumerate}
        \item a chain $C_n$ has dimension one;
        \item the Boolean lattice $B_m=(2^{[m]},\subseteq)$ has dimension $m$;
        \item if $N=\prod_{i=1}^r p_i^{e_i}$ with distinct primes and $e_i\geq1$, then the positive-divisor poset of $N$ has dimension $r$;
        \item if $D_n=([n],\mid)$ and the product of the first $r$ primes is at most $n$, then $\dimposet(D_n)\geq r$, so these dimensions are unbounded.
        \end{enumerate}
        \end{proposition}
        \begin{proof}
        A nonempty chain is already a linear order, so its dimension is one.
        
        For $B_m$, membership indicators give an $m$-coordinate representation, so Theorem~\ref{supp:coordinate} yields the upper bound. The case $m=1$ is a chain. The lattice $B_2$ is not a chain, so its dimension is at least two, matching the upper bound. Assume $m\geq3$ and define $a_i=\{i\}$ and $b_i=[m]\setminus\{i\}$. Then $a_i\subseteq b_j$ when $i\neq j$, while $a_i$ and $b_i$ are incomparable. A single linear extension cannot reverse both pairs $(a_i,b_i)$ and $(a_j,b_j)$ for $i\neq j$, because the required relations would give
        \[
        b_i<a_i<b_j<a_j<b_i.
        \]
        For each incomparable pair $(a_i,b_i)$, some member of any realizer must place $b_i$ before $a_i$; otherwise $a_i\leq b_i$ would survive in the intersection. Hence at least $m$ linear extensions are needed, proving $\dimposet(B_m)=m$.
        
        For divisors of $N$, map $d$ to the exponent vector
        $(v_{p_1}(d),\ldots,v_{p_r}(d))$. Divisibility is exactly coordinatewise comparison, giving dimension at most $r$. The squarefree divisors $\prod_{i\in A}p_i$ induce a copy of $B_r$. Order dimension is monotone under subposets because restricting every linear extension in a realizer gives a realizer of the subposet. Thus the divisor poset has dimension at least $r$ and therefore exactly $r$.
        
        Finally, if $M=p_1\cdots p_r\leq n$, every squarefree divisor of $M$ lies in $[n]$ and the induced divisibility subposet is $B_r$. Monotonicity gives $\dimposet(D_n)\geq r$. Since primorials are finite for every fixed $r$, the dimensions are unbounded as $n$ grows. Sharper asymptotics are known \cite{lewis2021divisibility}.
        \end{proof}
        
        \paragraph{Structural-profile entries.}
            For a chain $C_n$, the longest cover path and the largest reachable set associated with a false query both have size $n-1$. For $B_m$, covers add one ground-set element, giving $m2^{m-1}$ covers and maximum positive path length $m$; a false query with a singleton left set has $2^{m-1}$ reachable supersets, which is maximal. For $N=\prod_i p_i^{e_i}$, exponent vectors show height $1+\sum_i e_i$, cover count $\sum_i e_i\prod_{j\neq i}(e_j+1)$, maximum positive path length $\sum_i e_i$, and maximum negative witness size $\max_i e_i\prod_{j\neq i}(e_j+1)$. For $D_n=([n],\mid)$, the chain $1,2,4,\ldots$ gives height $\lfloor\log_2 n\rfloor+1$ and maximum positive length $\lfloor\log_2 n\rfloor$; the numbers greater than $n/2$ form an antichain of size $\lceil n/2\rceil$; covers are exactly pairs $(a,ap)$ with $p$ prime, giving $\sum_{p\leq n}\lfloor n/p\rfloor$; and a false query from $a=2$ has $\lfloor n/2\rfloor$ reachable multiples, which is maximal among $a\geq2$.
    
    \section{Expanded Certificate Proofs}
        Let $H(P)$ be the complete Hasse DAG of a finite poset. For $a\preceq_P b$, let $\lambda_P^+(a,b)$ be the length of the shortest directed cover path from $a$ to $b$, with length zero when $a=b$.
    
        \begin{proposition}[Positive chain-certificate bound]\label{supp:positive-cert}
            Any positive certificate that consists only of demonstrated cover edges has at least $\lambda_P^+(a,b)$ edges. A procedure required to produce or verify a cover path of length at most $T$ therefore cannot certify every positive query with $\lambda_P^+(a,b)>T$.
        \end{proposition}
        \begin{proof}
            A certificate made only of cover edges is a directed path in the Hasse DAG. By definition, $\lambda_P^+(a,b)$ is the minimum number of edges among all such paths. No shorter valid chain certificate exists. The second statement follows from the restriction on the permitted witness. It does not rule out a different global reachability algorithm.
        \end{proof}
    
        Call $S\subseteq U$ forward closed if no Hasse edge has its tail in $S$ and its head outside $S$.
    
        The following elementary graph observation is recorded because it yields the canonical negative witness.
        \begin{lemma}[Forward-closed nonreachability witness]\label{supp:negative}
            There is no directed path from $a$ to $b$ in a finite Hasse DAG if and only if there exists a forward-closed set $S$ with $a\in S$ and $b\notin S$.
        \end{lemma}
        \begin{proof}
            If such an $S$ exists, a directed path starting at $a\in S$ cannot leave $S$, because every traversed edge remains inside a forward-closed set. It therefore cannot reach $b\notin S$.
        
            Conversely, suppose $b$ is not reachable from $a$ and take $S=\Reach_H(a)$. Then $a\in S$ and $b\notin S$. If an edge $(x,y)$ left $S$, then $x$ would be reachable from $a$ and the additional edge would make $y$ reachable as well, contradicting $y\notin S$. Hence $S$ is forward closed.
        \end{proof}
    
        For a nonreachable pair define the canonical witness-size statistic directly by
        \[
            \nu_P^-(a,b)=|\Reach_H(a)|.
        \]
    
        \begin{corollary}[Minimality of the canonical negative witness]\label{supp:negative-min}
            If $b$ is not reachable from $a$, then $\Reach_H(a)$ is the unique inclusion-minimal forward-closed set containing $a$ and excluding $b$. It also has minimum cardinality among such sets. Thus the direct definition of $\nu_P^-(a,b)$ equals the corresponding minimum, and for fixed $H$ and $a$ it is independent of the choice of nonreachable $b$.
        \end{corollary}
        \begin{proof}
            Lemma~\ref{supp:negative} shows that $\Reach_H(a)$ is an eligible separator. Let $S$ be any forward-closed set containing $a$. We prove by induction on path length that every vertex reachable from $a$ lies in $S$. The base vertex $a$ lies in $S$. If a path reaches $y$ through an edge $(x,y)$ and the induction hypothesis gives $x\in S$, forward closure gives $y\in S$. Thus $\Reach_H(a)\subseteq S$ for every eligible separator. It is therefore the unique inclusion-minimal one and has minimum cardinality.
        \end{proof}
    
    \section{Assumptions and Dependency Summary}
        The completion and teaching results depend on Lemma~S3's one-edge closure formula. The trichotomy assumes all posets on a fixed finite universe and a satisfiable mixed-label prompt. The deterministic illustration is an exhaustive enumeration for the specifically defined uniform distribution over four-element targets, prompt subsets, and unobserved queries; no theorem depends on that illustration. The teaching theorem uses the same known universe and arbitrary positive and negative ordered-pair labels. Corollary~S11 concerns exact prompt-dependent monotone-coordinate decoding, and the certificate results concern the stated path and forward-closed witness systems.
    
        Version spaces, teaching dimension, Hasse reachability, and Dushnik--Miller dimension are classical \cite{mitchell1977version,goldman1995teaching,dushnik1941posets,trotter1992dimension}. The new synthesis separates prompt entailment, the open-world teaching surcharge, structural profile coordinates, certificate size, and the exact capacity of a specified coordinate-decoder class.

\end{document}